\documentclass{article}

\usepackage{spconf,amsmath,amssymb,amsthm,graphicx}
\usepackage{algorithm}
\usepackage{algorithmic}
\usepackage{booktabs}
\usepackage{multirow}
\usepackage{cite}
\usepackage{color,hyperref}
\usepackage[normalem]{ulem}
\newcommand{\cN}{\mathcal{N}}
\newcommand{\cR}{\mathcal{R}}
\newcommand{\cP}{\mathcal{P}}
\newcommand{\cW}{\mathcal{W}}

\newtheorem{definition}{Definition}
\newtheorem{lemma}{Lemma}
\newtheorem{theorem}{Theorem}
\newtheorem{assumption}{Assumption}

\newif\ifwithappendix
\withappendixtrue

\title{FairMean: Promoting Fairness in Distributed Learning\\ under Label Poisoning Attacks}

\name{
	Huigan Zheng\textsuperscript{1,2},
	Jiaojiao Zhang\textsuperscript{3},
	Yongxiang Liu\textsuperscript{2}
\thanks{{Jiaojiao Zhang is supported by the National Natural Science Foundation of China under Grants 12601589. Corresponding author: Huigan Zheng.}}}

\address{
	\textsuperscript{1}Sun Yat-Sen University,
    \textsuperscript{2}Pengcheng Laboratory,
	\textsuperscript{3}Great Bay University	
}

\begin{document}	
	\maketitle	
	\begin{abstract}
		Fairness-aware distributed learning prioritizes clients with large losses to reduce performance disparities, but label poisoning can create  large losses, thereby inducing a fairness--robustness conflict. We propose FairMean to manage this conflict. 
FairMean weights client gradients using a bounded, nondecreasing function of local loss. The increasing weights prioritize high-loss clients to promote fairness, while the upper bound prevents excessive loss-induced amplification of poisoned-client gradients. In the absence of label poisoning, we show that minimizing the FairMean objective is more conducive to solution fairness than minimizing the standard average-loss objective.
Under label poisoning, we establish an average-stationarity bound whose attack-dependent term is proportional to the square of the poisoned-client fraction. Experiments show that FairMean promotes fairness by reducing accuracy variance while improving worst-client accuracy.		
	\end{abstract}	
		\vspace{-2mm}
	\section{Introduction}
	\label{sec:introduction}
	
	Distributed learning trains a global model from the private data of multiple clients through the coordination of a central server \cite{mcmahan2017fedavg,li2020federated,ren2025advances}. The standard formulation minimizes the average of the local losses. When client data are heterogeneous, the jointly trained model may perform well for most clients but poorly for a minority of clients \cite{wang2021heterogeneous,shi2023fairness}. Because a larger client loss generally indicates poorer current performance, fairness-aware formulations prioritize clients with large losses to narrow performance disparities across clients \cite{mukhtiar2025fairness}.
	
	\begin{figure*}[htbp]
		\centering
		\includegraphics[width=0.8\textwidth]{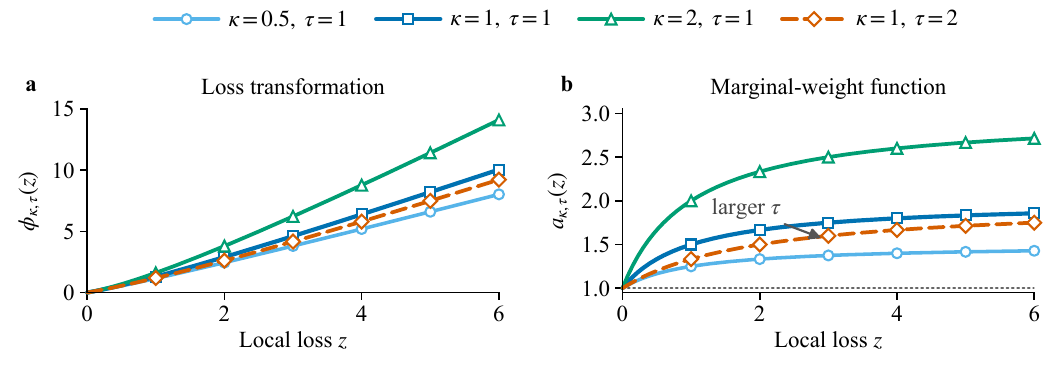}
		\vspace{-2.0em}
		\caption{Loss transformation and marginal-weight function. (a) The increment $\phi_{\kappa,\tau}(z)-z$ grows with $z$, so larger losses undergo a stronger transformation. (b) The marginal weight $a_{\kappa,\tau}(z)$ is nondecreasing and approaches $1+\kappa$. Solid curves fix $\tau=1$ and vary $\kappa$; the dashed curve uses $(\kappa,\tau)=(1,2)$. For $\kappa=1$, the blue point $(z,a_{\kappa,\tau})=(1,1.5)$ on the $\tau=1$ curve and the orange diamond $(2,1.5)$ on the $\tau=2$ curve have the same marginal weight; increasing $\tau$ shifts the required loss from $z=1$ to $z=2$.}
		\label{fig:loss_transformation}
	\end{figure*}
	
	Fairness-aware methods mainly follow two optimization routes. The first transforms or reweights client losses: $q$-FFL~\cite{li2020qffl} uses a power transformation, DRFL~\cite{zhao2022drfl} adapts aggregation weights to client losses, and TERM~\cite{li2023term} applies exponential tilting. The second treats local losses as multiple objectives. FedMDFG~\cite{pan2023fedmdfg} and FedLF~\cite{pan2024fedlf} introduce dynamically adjusted fairness guidance objectives. AdaFed~\cite{hamidi2024adafed} rescales client gradients before computing a common descent direction, while FairMOO~\cite{zheng2026fairmoo} alternates fairness reduction with constrained multi-objective descent.
	
Label poisoning attacks corrupt the local labels of a subset of clients while leaving their features and the prescribed protocol unchanged~\cite{jebreel2024lfighter,peng2025mean}; poisoned clients thus transmit misleading updates that raise their local losses and may degrade the global model on clean data. Label poisoning can be viewed as a special case of Byzantine attacks, where poisoned clients follow the prescribed protocol rather than transmitting arbitrary messages~\cite{lamport1982byzantine}; robust aggregators designed for the latter screen updates via robust statistics or outlier detection~\cite{yin2018byzantine,pillutla2022robust,xia2019faba,molodtsov2026bant}, but as we discuss below, such screening can be ineffective under label poisoning on heterogeneous data.

	Recent studies incorporate robustness mechanisms into fairness-aware learning. For example, the $q$-FFL instantiation of H-nobs combines a fairness objective with post-hoc norm-based screening~\cite{li2020qffl,zhou2023hnobs}. For $q>0$, the $q$-FFL multiplier increases without an explicit upper bound, whereas screening in H-nobs can discard fairness-enhancing gradients under heterogeneity. FedMGDA+~\cite{hu2022fedmgda} instead normalizes client updates before computing a common descent direction for multiple objectives. Neither mechanism directly resolves the fairness--poisoning conflict: screening may discard atypical regular-client updates under heterogeneity, whereas normalization does not provide loss-dependent prioritization. Consequently, how fairness-aware methods behave when poisoned clients follow the prescribed computation and communication rules remains unclear.

	Recent work~\cite{peng2025mean} shows that the mean aggregator can outperform several robust aggregators under label poisoning attacks when client data are sufficiently heterogeneous. Intuitively, robust filtering may discard useful heterogeneous updates, whereas mean aggregation retains them. These observations favor retaining mean aggregation in heterogeneous settings. Plain mean aggregation, however, gives all client gradients equal weights and does not prioritize clients with large losses. Loss-dependent gradient scaling can provide such prioritization, but it creates an inherent conflict: underperforming regular clients and label-poisoned clients may both exhibit large losses. An unbounded scaling rule may therefore amplify gradients induced by corrupted labels. This raises a central question: \textbf{how can we prioritize regular clients with large losses without excessively scaling gradients from poisoned clients?}
	
	To answer this question, we propose FairMean. Our contributions are: \emph{(i)} a loss transformation with a nondecreasing and uniformly bounded marginal weight that prioritizes large-loss clients to promote fairness while limiting, rather than eliminating, extra loss-induced scaling of poisoned-client gradients, and a clean-case global-optimality guarantee for the proposed loss-dispersion measure; \emph{(ii)} an explicit bound on the deviation from the gradient of the objective induced by the loss transformation over regular clients and a corresponding average-stationarity guarantee; and \emph{(iii)} experiments on Fashion-MNIST and CIFAR-10 showing improved fairness and worst-client accuracy under both attacks, with competitive average accuracy. 
	
	\section{Problem Formulation}
	\label{sec:problem}
	

	Consider $n$ clients indexed by $\cN=\{1,\ldots,n\}$ and a global model $w\in\mathbb R^d$. Client $i\in\cN$ has a differentiable local loss $f_i:\mathbb R^d\to[0,\infty)$. Every client is either regular or poisoned. Let $\cR$ and $\cP$ denote the corresponding fixed but unknown index sets. We write $r:=|\cR|\ge1$ and $p:=|\cP|$, so that $r+p=n$, and define the poisoned-client fraction as $\delta:=p/n<1$. The nominal objective over the regular clients is
	\begin{equation}
		\min_{w\in\mathbb R^d}F_{\cR}(w),\quad
		F_{\cR}(w):=\frac1r\sum_{i\in\cR}f_i(w).
		\label{eq:regular_problem}
	\end{equation}
	
	\begin{definition}
		\label{def:label_poisoning}
		Following \cite{peng2025mean}, client $i\in\cP$ changes an arbitrary subset of its labels while leaving all other data unchanged. It computes $\widetilde f_i$ on the relabeled data and follows the prescribed protocol; it cannot transmit an arbitrary message.
	\end{definition}
	
	The transmitted update therefore changes only through relabeling, although corrupted labels may still produce a large loss and a misleading gradient.

	
	\begin{definition}
		\label{def:regular_client_fairness}
		Following \cite{li2020qffl}, fairness concerns the uniformity of model performance across clients. Under label poisoning, we assess fairness over the regular clients $\cR$.
	\end{definition}
	
	Regular-client loss dispersion serves as a training-time proxy for performance disparity. Since a large regular-client loss often indicates poor current performance, assigning it a larger marginal weight can promote fairness. However, poisoned labels may also cause large losses, so unrestricted prioritization can amplify poisoned gradients. This conflict motivates Section~\ref{sec:method}.
	\vspace{-2mm}
	\section{Proposed FairMean}
	\label{sec:method}
Label poisoning changes labels while leaving client features and the prescribed training procedure unchanged, yet it can produce large losses and misleading gradients. Since mean aggregation can outperform robust alternatives under such attacks in heterogeneous settings~\cite{peng2025mean}, we propose FairMean, which retains mean aggregation and introduces a bounded, nondecreasing marginal weight to prioritize large losses while limiting the amplification of poisoned-client gradients.
    
	
	We seek a differentiable loss transformation whose derivative is \emph{(i)} nondecreasing to prioritize large regular-client losses, \emph{(ii)} uniformly bounded to cap extra loss-induced scaling of every client gradient, including gradients produced from poisoned labels, and \emph{(iii)} Lipschitz continuous to avoid abrupt weight changes. These requirements lead to
	\begin{equation}
		\phi_{\kappa,\tau}(z):=z+\kappa\left[z-\tau\log(1+z/\tau)\right], \quad \kappa,\tau>0.
		\label{eq:fair_transform}
	\end{equation}
	For analysis, define the regular-client objective induced by the loss transformation as
	\begin{equation}
		\Phi_{\cR}^{\kappa,\tau}(w)
		:=\frac{1}{r}\sum_{i\in\cR}\phi_{\kappa,\tau}(f_i(w)).
		\label{eq:fair_objective}
	\end{equation}
	To characterize the gradient of $\Phi_{\cR}^{\kappa,\tau}$, define the corresponding marginal-weight function as
	\begin{equation}
		a_{\kappa,\tau}(z):=1+\kappa\frac{z}{z+\tau}.
		\label{eq:bounded_weight}
	\end{equation}
	By the chain rule,
	\begin{equation}
		\nabla\Phi_{\cR}^{\kappa,\tau}(w)
		=
		\frac1r\sum_{i\in\cR}a_{\kappa,\tau}(f_i(w))\nabla f_i(w).
	\end{equation}
	Thus, $a_{\kappa,\tau}(f_i(w))$ is client $i$'s \emph{marginal weight}, namely, the scalar that multiplies $\nabla f_i(w)$.
	For $z\ge0$, $\phi_{\kappa,\tau}(z)\ge0$. The marginal weight and
	its derivative satisfy
	\[
	\begin{aligned}
		1 &\le a_{\kappa,\tau}(z)\le1+\kappa,\\
		0 &\le a_{\kappa,\tau}'(z)
		=\frac{\kappa\tau}{(z+\tau)^2}
		\le\frac{\kappa}{\tau}.
	\end{aligned}
	\]
	Thus, $a_{\kappa,\tau}$ is nondecreasing, bounded by $1+\kappa$, and $\kappa/\tau$-Lipschitz continuous. Moreover, $a_{\kappa,\tau}(\tau)=1+\kappa/2$, so increasing $\tau$ shifts the midpoint to a larger loss. The cap $1+\kappa$ directly limits the additional loss-induced scaling.
	
	To connect the loss transformation with fairness, define \(\psi_\tau(z):=z-\tau\log(1+z/\tau)\) and the dispersion term
	\[
	D_\tau(w):=\frac1r\sum_{i\in\cR}\psi_\tau(f_i(w))
	-\psi_\tau(F_{\cR}(w)).
	\]
	Substituting the loss transformation in~\eqref{eq:fair_transform} into~\eqref{eq:fair_objective} shows that $\Phi_{\cR}^{\kappa,\tau}(w)$ in~\eqref{eq:fair_objective} can be equivalently written as
	\begin{equation}
		\Phi_{\cR}^{\kappa,\tau}(w)
		=\phi_{\kappa,\tau}(F_{\cR}(w))+\kappa D_\tau(w).
		\label{eq:objective_decomposition}
	\end{equation}
	The first term depends only on the mean regular-client loss, while $D_\tau(w)\ge0$ measures loss dispersion and vanishes exactly when all regular-client losses are equal. 
	
	Moreover, $\psi_\tau''(z)=\tau/(z+\tau)^2$. If $0\le f_i(w)\le B$ for all $i\in\cR$, a second-order Taylor expansion around $F_{\cR}(w)$, followed by averaging, cancels the linear terms and yields
	\begin{equation}
		\frac{\tau}{2(B+\tau)^2}\mathrm{Var}_{\cR}(w)
		\le D_\tau(w)\le
		\frac{1}{2\tau}\mathrm{Var}_{\cR}(w),
		\label{eq:dispersion_bounds}
	\end{equation}
	where $\mathrm{Var}_{\cR}(w):=r^{-1}\sum_{i\in\cR}[f_i(w)-F_{\cR}(w)]^2$. Thus, minimizing the FairMean objective explicitly penalizes dispersion among regular-client losses, linking the loss transformation to fairness.
		
	\begin{algorithm}[H]
		\caption{FairMean}
		\label{alg:fairmean}
		\begin{algorithmic}[1]
			\REQUIRE Initial model $w^0$, number of rounds $T$, stepsize $\alpha>0$, $\kappa>0$, and $\tau>0$
			\FOR{$t=0,\ldots,T-1$}
			\STATE Server broadcasts $w^t$
			\FOR{$i=1,\ldots,n$ \textbf{ in parallel}}
			\STATE Compute $\widehat f_i(w^t)$ and $\nabla\widehat f_i(w^t)$
			\STATE $v_i^t\gets a_{\kappa,\tau}(\widehat f_i(w^t))\nabla\widehat f_i(w^t)$
			\STATE Send $v_i^t$ to the server
			\ENDFOR
			\STATE $w^{t+1}\gets w^t-\alpha 
			\frac{1}{n}\sum_{i=1}^n v_i^t$
			\ENDFOR
			\RETURN $w^T$
		\end{algorithmic}
	\end{algorithm}
	
	
	Because the server does not know $\cR$, FairMean applies the loss transformation to every client. Let $\widehat f_i=f_i$ for $i\in\cR$ and $\widehat f_i=\widetilde f_i$ for $i\in\cP$, denoting the loss evaluated by client $i$. At round $t$, client $i$ transmits
	\begin{equation}
		v_i^t:=a_{\kappa,\tau}(\widehat f_i(w^t))
		\nabla\widehat f_i(w^t).
		\label{eq:client_message}
	\end{equation}
	Using a fixed stepsize $\alpha>0$, the server averages the received vectors and updates
	\begin{equation}
		w^{t+1}:=w^t-\frac{\alpha}{n}
		\sum_{i=1}^{n}v_i^t.
		\label{eq:server_update}
	\end{equation}

	\section{Theoretical Analysis}
	\label{sec:theory}
	We first show that, without label poisoning, minimizing the proposed FairMean objective in~\eqref{eq:fair_objective} is more conducive to solution fairness than minimizing the original objective in~\eqref{eq:regular_problem}. Throughout this paper, $\|\cdot\|$ denotes the Euclidean norm for vectors and the spectral norm for matrices. Proofs of all theoretical results are provided in	Appendix~\ref{app:proofs}.
	
	\begin{theorem}
		\label{thm:clean_dispersion}
		Assume $p=0$ and that $F_{\cR}$ and $\Phi_{\cR}^{\kappa,\tau}$ admit global minimizers $w_{\mathrm{avg}}^\star$ and $w_{\mathrm{FM}}^\star$, respectively. Then
		\[
		D_\tau(w_{\mathrm{FM}}^\star)
			\le D_\tau(w_{\mathrm{avg}}^\star).
		\]
	\end{theorem}
	
	Theorem~\ref{thm:clean_dispersion} shows that, without label poisoning, $w_{\mathrm{FM}}^\star$, a global minimizer of the objective in~\eqref{eq:fair_objective}, has no larger $D_\tau$ than $w_{\mathrm{avg}}^\star$, a global minimizer of the original objective $F_{\cR}$ in~\eqref{eq:regular_problem}. Together with \eqref{eq:dispersion_bounds}, the theorem provides a theoretical guarantee that minimizing the objective induced by $\phi_{\kappa,\tau}$ in~\eqref{eq:fair_objective} promotes fairness at global optimality. 

	We now analyze FairMean as inexact gradient descent with respect to $\Phi_{\cR}^{\kappa,\tau}$.
	
	\begin{assumption}
		\label{ass:smoothness}
		Let the convex set $\cW$ contain $w^0,\ldots,w^T$. For every $i\in\cR$, the regular-client loss $f_i$ is $L$-smooth on $\cW$.
	\end{assumption}

	\begin{assumption}
		\label{ass:boundedness}
		For some $G,A_{\mathrm{LP}}\ge0$ and every $w\in\cW$, the following bounds hold for $i\in\cR$ and $i\in\cP$, respectively:
		\[
		\|\nabla f_i(w)\|\le G,\quad
		\|\nabla\widetilde f_i(w)-\nabla F_{\cR}(w)\|
		\le A_{\mathrm{LP}}.
		\]
	\end{assumption}
	
\begin{table*}[!t]
		\centering
		\caption{CIFAR-10 test performance of regular clients in a 10-client network. Arrows indicate preferred directions.}
		\label{tab:main_results}
		\setlength{\tabcolsep}{4pt}
		\renewcommand{\arraystretch}{1.02}
		\begin{tabular}{lccc ccc ccc}
			\toprule
			\multirow{2}{*}{Method}
			& \multicolumn{3}{c}{No attack}
			& \multicolumn{3}{c}{Random flip}
			& \multicolumn{3}{c}{Pairwise flip}\\
			\cmidrule(lr){2-4}\cmidrule(lr){5-7}\cmidrule(lr){8-10}
			& Avg.$\uparrow$ & Fairness$\downarrow$ & Worst$\uparrow$
			& Avg.$\uparrow$ & Fairness$\downarrow$ & Worst$\uparrow$
			& Avg.$\uparrow$ & Fairness$\downarrow$ & Worst$\uparrow$\\
			\midrule
			FedAvg
			& \textbf{71.54} & 40.85 & 54.57
			& 60.59 & 177.70 & 37.62
			& 59.98 & 203.39 & 32.96\\
			$q$-FFL
			& 67.93 & 40.50 & \textbf{57.24}
			& 60.62 & 106.12 & 46.30
			& \underline{62.66} & 131.22 & 38.42\\
			AdaFed
			& \underline{70.89} & \textbf{31.83} & \underline{56.18}
			& 61.18 & 110.59 & 45.98
			& 61.71 & \underline{106.28} & \underline{50.53}\\
			FedMGDA+
			& 70.33 & 37.92 & 54.41
			& 55.90 & \underline{104.22} & 45.18
			& \textbf{63.09} & 118.47 & 49.04\\
			H-nobs
			& 67.93 & 40.50 & \textbf{57.24}
			& \textbf{63.38} & 108.06 & \underline{50.83}
			& 56.35 & 161.16 & 41.18\\
			$q$-FFL+CWTM
			& 67.93 & 40.50 & \textbf{57.24}
			& 32.71 & 406.58 & 0.49
			& 34.29 & 379.87 & 8.73\\
			FairMean
			& 70.15 & \underline{34.46} & 55.94
			& \underline{61.54} & \textbf{85.33} & \textbf{52.01}
			& 61.37 & \textbf{99.34} & \textbf{50.97}\\
			\bottomrule
		\end{tabular}
	\end{table*}

	We next show that Assumption~\ref{ass:boundedness} holds for multiclass cross-entropy under label poisoning.
	\begin{lemma}
		\label{lem:lp_instantiation}
		Suppose all pre- and post-poisoning local losses are empirical cross-entropy losses of a $K$-class model with logits $h(w;x)\in\mathbb R^K$. Assume that the logit Jacobian $\nabla_w h(w;x)\in\mathbb R^{K\times d}$ satisfies $\|\nabla_w h(w;x)\|\le\Gamma$ on $\cW$ for every local sample. Also assume $\|\nabla f_i(w)-\nabla F_{\cR}(w)\|\le H_{\cP}$ for every $i\in\cP$ and $w\in\cW$. Let $\rho\in[0,1]$ be the largest fraction of relabeled samples at any poisoned client. Then Assumption~\ref{ass:boundedness} holds with
		\begin{equation}
			G:=\sqrt{2}\Gamma,\quad
			A_{\mathrm{LP}}:=H_{\cP}+\sqrt{2}\,\rho\Gamma.
			\label{eq:explicit_alp}
		\end{equation}
	\end{lemma}
	
	\begin{lemma}
		\label{lem:inexact_direction}
		Under Assumption~\ref{ass:boundedness}, define $e^t:=n^{-1}\sum_{i=1}^n v_i^t-\nabla\Phi_{\cR}^{\kappa,\tau}(w^t)$. Then~\eqref{eq:server_update} can be written as
		\[
			w^{t+1}=w^t-\alpha\big[\nabla\Phi_{\cR}^{\kappa,\tau}(w^t)+e^t\big].
		\]
		Moreover, the error satisfies
		\begin{equation}
			\|e^t\|\le \epsilon(\delta,\kappa),\quad
			\epsilon(\delta,\kappa):=\delta\big[(1+\kappa)A_{\mathrm{LP}}+2\kappa G\big].
			\label{eq:effective_radius}
		\end{equation}
	\end{lemma}
	
	By~\eqref{eq:effective_radius}, the attack contribution to the stationarity bound below contains the explicit factor $\delta^2$.

	\begin{theorem}
		\label{thm:stationarity}
		Under Assumptions~\ref{ass:smoothness}--\ref{ass:boundedness}, let $L_\Phi(\kappa,\tau):=(1+\kappa)L+(\kappa/\tau)G^2$. For $0<\alpha\le1/L_\Phi(\kappa,\tau)$, FairMean satisfies
		\begin{equation}
			\frac1T\sum_{t=0}^{T-1} \|\nabla\Phi_{\cR}^{\kappa,\tau}(w^t)\|^2 \le \frac{2\Phi_{\cR}^{\kappa,\tau}(w^0)}{\alpha T} +\epsilon(\delta,\kappa)^2.
			\label{eq:stationarity_bound}
		\end{equation}
	\end{theorem}
	The optimization term in~\eqref{eq:stationarity_bound} decays as $O(1/T)$. Increasing $\kappa$ tightens the stepsize bound, whereas increasing $\tau$ relaxes it through the term $(\kappa/\tau)G^2$, consistent with Fig.~\ref{fig:loss_transformation}(b). Thus, more aggressive prioritization of high-loss clients, induced by a larger $\kappa$ or a smaller $\tau$, requires a more conservative theoretical stepsize.
		\vspace{-1mm}
	
	\begin{figure}[t]
		\centering
		\includegraphics[width=0.95\columnwidth]{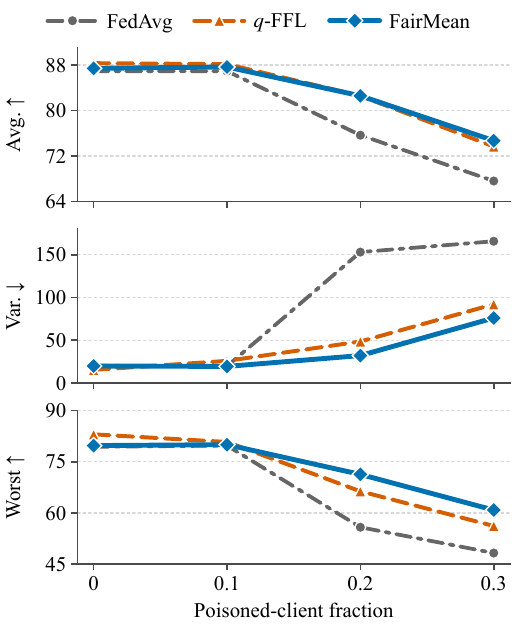}
		\vspace{-2.0em}
		\caption{Fashion-MNIST performance under pairwise flip versus the poisoned-client fraction.}
		\label{fig:attack_fraction}
	\end{figure}
	\vspace{-1mm}
	\section{Experiments}
	\label{sec:experiments}
		\vspace{-1mm}
	\noindent\textbf{Experimental setup.} We use 10 clients, an MLP for Fashion-MNIST, and a 5-layer CNN for CIFAR-10\footnote{\label{code} \url{https://github.com/Zhg9300/FairnessUnderLP}}. Their training samples follow Dirichlet partitions with concentrations 0.5 and 0.1, respectively. The main experiments poison two clients. Random flip selects an incorrect label uniformly, while pairwise flip maps class $c$ to $9-c$. Baselines are FedAvg~\cite{mcmahan2017fedavg}, $q$-FFL~\cite{li2020qffl}, AdaFed~\cite{hamidi2024adafed}, FedMGDA+~\cite{hu2022fedmgda}, H-nobs~\cite{zhou2023hnobs}, and $q$-FFL+CWTM~\cite{yin2018byzantine}.
	
	\noindent\textbf{Evaluation metrics.} Using clean test labels, we evaluate regular clients by accuracy variance, worst-client accuracy, and average accuracy, which measure fairness, poorest-client performance, and overall performance, respectively.
	
	\noindent\textbf{Overall performance.} Tables~\ref{tab:main_results} and~\ref{tab:fashion_results} show that FairMean remains competitive without attacks. Under label poisoning, it achieves the lowest accuracy variance and highest worst-client accuracy in all four dataset--attack settings, while retaining competitive average accuracy.
	
	\noindent\textbf{Impact of the poisoned-client fraction.} On Fashion-MNIST under pairwise flip, we vary the poisoned-client fraction from 0 to 0.3. Figure~\ref{fig:attack_fraction} shows that FairMean retains its fairness advantage. At $\delta=0.3$, it reduces accuracy variance from $91.99$ to $75.97$ and raises worst-client accuracy from $56.19\%$ to $60.85\%$ relative to $q$-FFL.

	\noindent\textbf{Future work.}
	Future work will extend FairMean to stochastic client participation and develop mechanisms that account for benign data heterogeneity when mitigating label poisoning.
	
	\clearpage
	\raggedbottom
	\bibliographystyle{IEEEbib}
	\bibliography{strings,refs}

	\ifwithappendix
	\clearpage
	\flushbottom
	\appendix

	\section{Additional Experimental Results}
	\label{app:experiments}
	\begin{table*}[t]
		\centering
		\caption{Fashion-MNIST test performance of regular clients in a 10-client network. Arrows indicate preferred directions.}
		\label{tab:fashion_results}
		\setlength{\tabcolsep}{4pt}
		\renewcommand{\arraystretch}{1.02}
		\begin{tabular}{lccc ccc ccc}
			\toprule
			\multirow{2}{*}{Method}
			& \multicolumn{3}{c}{No attack}
			& \multicolumn{3}{c}{Random flip}
			& \multicolumn{3}{c}{Pairwise flip}\\
			\cmidrule(lr){2-4}\cmidrule(lr){5-7}\cmidrule(lr){8-10}
			& Avg.$\uparrow$ & Fairness$\downarrow$ & Worst$\uparrow$
			& Avg.$\uparrow$ & Fairness$\downarrow$ & Worst$\uparrow$
			& Avg.$\uparrow$ & Fairness$\downarrow$ & Worst$\uparrow$\\
			\midrule
			FedAvg
			& 86.90 & 20.00 & 79.38
			& 77.08 & 141.43 & 58.23
			& 75.65 & 153.12 & 55.82\\
			$q$-FFL
			& \underline{88.35} & 16.11 & \underline{81.33}
			& 81.69 & 73.66 & 65.97
			& 82.56 & \underline{48.41} & 66.34\\
			AdaFed
			& 85.62 & \textbf{12.46} & 79.20
			& 74.55 & 129.47 & 55.20
			& 73.10 & 139.15 & 54.95\\
			FedMGDA+
			& \textbf{88.62} & \underline{15.49} & \textbf{81.96}
			& 73.68 & 126.05 & 56.93
			& 76.93 & 155.32 & 56.19\\
			H-nobs
			& \underline{88.35} & 16.11 & \underline{81.33}
			& \textbf{84.19} & \underline{56.79} & \underline{69.60}
			& \textbf{83.99} & 67.89 & \underline{67.47}\\
			$q$-FFL+CWTM
			& \underline{88.35} & 16.11 & \underline{81.33}
			& 54.24 & 431.55 & 19.09
			& 63.07 & 466.48 & 19.37\\
			FairMean
			& 87.55 & 18.44 & 80.36
			& \underline{81.87} & \textbf{42.12} & \textbf{70.55}
			& \underline{82.57} & \textbf{31.95} & \textbf{71.29}\\
			\bottomrule
		\end{tabular}
	\end{table*}
	\noindent Table~\ref{tab:fashion_results} shows that FairMean remains competitive without attacks. Under both random and pairwise flips, it achieves the lowest accuracy variance and the highest worst-client accuracy while maintaining competitive average accuracy, consistent with the main results on CIFAR-10.

	\section{Supporting Proofs}
	\label{app:proofs}
	\emph{Proof of Theorem~\ref{thm:clean_dispersion}.}
	Since $\phi_{\kappa,\tau}'(z)=a_{\kappa,\tau}(z)\ge1$, the loss transformation is strictly increasing. The optimality of $w_{\mathrm{avg}}^\star$ therefore implies
		$\phi_{\kappa,\tau}(F_{\cR}(w_{\mathrm{FM}}^\star))
		\ge
		\phi_{\kappa,\tau}(F_{\cR}(w_{\mathrm{avg}}^\star))$.
		The optimality of $w_{\mathrm{FM}}^\star$ and
		\eqref{eq:objective_decomposition} further give
	\[
	\begin{aligned}
			&\kappa\!\left[
			D_\tau(w_{\mathrm{avg}}^\star)-D_\tau(w_{\mathrm{FM}}^\star)
			\right]\\
			&=\Phi_{\cR}^{\kappa,\tau}(w_{\mathrm{avg}}^\star)
			-\Phi_{\cR}^{\kappa,\tau}(w_{\mathrm{FM}}^\star)\\
			&\quad+
			\phi_{\kappa,\tau}(F_{\cR}(w_{\mathrm{FM}}^\star))
			-\phi_{\kappa,\tau}(F_{\cR}(w_{\mathrm{avg}}^\star))
			\ge0.
	\end{aligned}
	\]
	Dividing by $\kappa>0$ proves the result.
	\hfill$\square$
	
	\emph{Proof of Lemma~\ref{lem:lp_instantiation}.}
	For each $i\in\cR$, write its local data as $\{(x_{ij},y_{ij})\}_{j=1}^{m_i}$. For each $i\in\cP$, write its data before and after poisoning as $\{(x_{ij},y_{ij},\widetilde y_{ij})\}_{j=1}^{m_i}$. The largest within-client relabeling fraction is
	\[
	\rho:=\max_{i\in\cP}\frac{1}{m_i}
	\bigl|\{j:\widetilde y_{ij}\ne y_{ij}\}\bigr|.
	\]
	Define $\pi(w;x):=\operatorname{softmax}(h(w;x))$ and $\ell(w;x,y):=-\log \pi_y(w;x)$. For any local sample,
	\[
	\nabla_w\ell(w;x,y)=\nabla_w h(w;x)^\top\bigl(\pi(w;x)-e_y\bigr),
	\]
	where $e_y$ is the $y$-th standard basis vector in $\mathbb R^K$. Since $\|\pi(w;x)-e_y\|\le\sqrt{2}$, the Jacobian bound gives
	\[
	\|\nabla_w\ell(w;x,y)\|
	\le\|\nabla_w h(w;x)\|\,\|\pi(w;x)-e_y\|
	\le\sqrt{2}\Gamma.
	\]
	Averaging over the local samples yields
	\[
	\|\nabla f_i(w)\|\le\sqrt{2}\Gamma,\quad i\in\cR,\ w\in\cW.
	\]
	Thus, the regular-client gradient bound holds with $G=\sqrt{2}\Gamma$. Fix $i\in\cP$. Because protocol-following label poisoning leaves $x$ unchanged, the softmax terms cancel and
	\[
	\nabla_w\ell(w;x,\widetilde y)-\nabla_w\ell(w;x,y)
		=\nabla_w h(w;x)^\top(e_y-e_{\widetilde y}).
	\]
	For every $j$ satisfying $\widetilde y_{ij}\ne y_{ij}$, the spectral norm bound gives
	\[
	\begin{aligned}
			&\left\|
			\nabla_w h(w;x_{ij})^\top
			\bigl(e_{y_{ij}}-e_{\widetilde y_{ij}}\bigr)
			\right\|\\
			&\le
			\|\nabla_w h(w;x_{ij})\|
			\|e_{y_{ij}}-e_{\widetilde y_{ij}}\|
			\le \sqrt{2}\Gamma.
	\end{aligned}
	\]
	Averaging over the local samples of client $i$, with unchanged labels contributing zero, yields
	\[	
		\|\nabla\widetilde f_i(w)-\nabla f_i(w)\|
		\le\frac{\sqrt{2}\Gamma}{m_i}
		|\{j:\widetilde y_{ij}\ne y_{ij}\}|
		\le\sqrt{2}\rho\Gamma.
	\]
	
	Adding and subtracting $\nabla f_i(w)$ therefore yields
	\[
	\begin{aligned}
			\|\nabla\widetilde f_i(w)-\nabla F_{\cR}(w)\|
			&\le \|\nabla\widetilde f_i(w)-\nabla f_i(w)\|\\
			&\quad+\|\nabla f_i(w)-\nabla F_{\cR}(w)\|\\
			&\le \sqrt{2}\,\rho\Gamma+H_{\cP},
	\end{aligned}
	\]
	which proves~\eqref{eq:explicit_alp}.
	\hfill$\square$
	
	\emph{Proof of Lemma~\ref{lem:inexact_direction}.}
	Fix $t\in\{0,\ldots,T-1\}$. For this proof, define the averaged
	server direction
	\[
	u^t:=\frac1n\sum_{i=1}^n v_i^t,
	\]
	so that the main-text definition gives $e^t=u^t-\nabla\Phi_{\cR}^{\kappa,\tau}(w^t)$. Also define
	\[
	g_{\cR}^t:=\nabla F_{\cR}(w^t)
	=\frac1r\sum_{i\in\cR}\nabla f_i(w^t).
	\]
	Assumption~\ref{ass:boundedness} and the triangle inequality imply
	\[
	\|g_{\cR}^t\|
	\le\frac1r\sum_{i\in\cR}\|\nabla f_i(w^t)\|
	\le G.
	\]
	For each $i\in\cR$, let
	$s_i^t:=f_i(w^t)/(f_i(w^t)+\tau)$. Since
	$a_{\kappa,\tau}(f_i(w^t))=1+\kappa s_i^t$ and
	$0\le s_i^t\le1$, we obtain
	\[
	\begin{aligned}
		\nabla\Phi_{\cR}^{\kappa,\tau}(w^t)
		&=g_{\cR}^t+
		\frac{\kappa}{r}\sum_{i\in\cR}
		s_i^t\nabla f_i(w^t),\\
		\|\nabla\Phi_{\cR}^{\kappa,\tau}(w^t)-g_{\cR}^t\|
		&\le\frac{\kappa}{r}\sum_{i\in\cR}
		s_i^t\|\nabla f_i(w^t)\|
		\le\kappa G.
	\end{aligned}
	\]
	
	We first consider $p=0$. In this case, $n=r$ and $\delta=0$, and
	the definitions of $v_i^t$ and $u^t$ give
	\[
	u^t=\frac1r\sum_{i\in\cR}
	a_{\kappa,\tau}(f_i(w^t))\nabla f_i(w^t)
	=\nabla\Phi_{\cR}^{\kappa,\tau}(w^t).
	\]
	Thus, $e^t=0$, and \eqref{eq:effective_radius} holds.
	
	Suppose next that $p>0$. For $i\in\cP$, define
	\[
	\widetilde a_i^t:=
	a_{\kappa,\tau}(\widetilde f_i(w^t)),
	\quad
	q_{\cP}^t:=\frac1p\sum_{i\in\cP}
	\widetilde a_i^t\nabla\widetilde f_i(w^t).
	\]
	The nonnegativity of $\widetilde f_i$ implies
	$1\le\widetilde a_i^t\le1+\kappa$ and
	$0\le\widetilde a_i^t-1\le\kappa$. Adding and subtracting
	$\widetilde a_i^t g_{\cR}^t$ within each summand gives
	\[
	\begin{aligned}
		q_{\cP}^t-g_{\cR}^t
		={}&\frac1p\sum_{i\in\cP}\widetilde a_i^t
		\bigl(\nabla\widetilde f_i(w^t)-g_{\cR}^t\bigr)\\
		&+\frac1p\sum_{i\in\cP}
		(\widetilde a_i^t-1)g_{\cR}^t.
	\end{aligned}
	\]
	Therefore, Assumption~\ref{ass:boundedness} yields
	\[
	\begin{aligned}
		\|q_{\cP}^t-g_{\cR}^t\|
		&\le\frac1p\sum_{i\in\cP}\widetilde a_i^t
		\|\nabla\widetilde f_i(w^t)-g_{\cR}^t\|\\
		&\quad+\frac1p\sum_{i\in\cP}
		(\widetilde a_i^t-1)\|g_{\cR}^t\|\\
		&\le(1+\kappa)A_{\mathrm{LP}}+\kappa G.
	\end{aligned}
	\]
	
	The regular-client contribution equals $r\nabla\Phi_{\cR}^{\kappa,\tau}(w^t)$.
	The identities $r/n=1-\delta$ and $p/n=\delta$ therefore give
	\[
	\begin{aligned}
		u^t
		&=\frac{r}{n}\nabla\Phi_{\cR}^{\kappa,\tau}(w^t)
		+\frac{p}{n}q_{\cP}^t\\
		&=(1-\delta)\nabla\Phi_{\cR}^{\kappa,\tau}(w^t)
		+\delta q_{\cP}^t.
	\end{aligned}
	\]
	Subtracting $\nabla\Phi_{\cR}^{\kappa,\tau}(w^t)$ from both sides and using the definition of $e^t$ yields
	\[
	e^t=\delta
	\bigl(q_{\cP}^t-\nabla\Phi_{\cR}^{\kappa,\tau}(w^t)\bigr).
	\]
	Combining the preceding two bounds gives
	\[
	\begin{aligned}
		\|e^t\|
		&\le\delta\bigl(
		\|q_{\cP}^t-g_{\cR}^t\|
		+\|g_{\cR}^t-\nabla\Phi_{\cR}^{\kappa,\tau}(w^t)\|
		\bigr)\\
		&\le\delta\big[(1+\kappa)A_{\mathrm{LP}}+2\kappa G\big]
		=\epsilon(\delta,\kappa).
	\end{aligned}
	\]
	This proves \eqref{eq:effective_radius} and completes the proof.
	\hfill$\square$
	
	\emph{Proof of Theorem~\ref{thm:stationarity}.}
	We first establish the smoothness of $\Phi_{\cR}^{\kappa,\tau}$. For any
	$x,y\in\cW$, convexity of $\cW$ and the regular-client gradient bound
	give
	\[
	\begin{aligned}
		|f_i(y)-f_i(x)|
		&=\left|\int_0^1
		\left\langle\nabla f_i(x+s(y-x)),y-x\right\rangle ds\right|\\
		&\le G\|y-x\|.
	\end{aligned}
	\]
	For brevity, set
	$a_i(w):=a_{\kappa,\tau}(f_i(w))$. The derivative bound following
	\eqref{eq:bounded_weight} gives
	\[
	|a_i(x)-a_i(y)|
	\le\frac{\kappa}{\tau}|f_i(x)-f_i(y)|
	\le\frac{\kappa G}{\tau}\|x-y\|.
	\]
	Define
	$h_i(w):=a_i(w)\nabla f_i(w)$. Adding and subtracting
	$a_i(x)\nabla f_i(y)$ and using $1\le a_i(x)\le1+\kappa$ gives
	\[
	\begin{aligned}
		\|h_i(x)-h_i(y)\|
		&\le a_i(x)
		\|\nabla f_i(x)-\nabla f_i(y)\|\\
		&\quad+|a_i(x)-a_i(y)|
		\|\nabla f_i(y)\|\\
		&\le\left[(1+\kappa)L+
		\frac{\kappa}{\tau}G^2\right]\|x-y\|\\
		&=L_\Phi(\kappa,\tau)\|x-y\|.
	\end{aligned}
	\]
	Since $\nabla\Phi_{\cR}^{\kappa,\tau}(w)=r^{-1}\sum_{i\in\cR}h_i(w)$, averaging
	this bound shows that $\Phi_{\cR}^{\kappa,\tau}$ is $L_\Phi(\kappa,\tau)$-smooth on $\cW$.
	
	For each $t=0,\ldots,T-1$, let $g^t:=\nabla\Phi_{\cR}^{\kappa,\tau}(w^t)$. By
	Lemma~\ref{lem:inexact_direction}, the update satisfies
	$w^{t+1}=w^t-\alpha(g^t+e^t)$. The segment between $w^t$ and
	$w^{t+1}$ lies in $\cW$, so the descent lemma gives
	\[
	\begin{aligned}
		\Phi_{\cR}^{\kappa,\tau}(w^{t+1})
		&\le\Phi_{\cR}^{\kappa,\tau}(w^t)
		-\alpha\langle g^t,g^t+e^t \rangle\\
		&\quad+\frac{L_\Phi(\kappa,\tau)\alpha^2}{2}\|g^t+e^t\|^2\\
		&\le\Phi_{\cR}^{\kappa,\tau}(w^t)
		-\alpha\langle g^t,g^t+e^t\rangle\\
		&\quad+\frac{\alpha}{2}\|g^t+e^t\|^2,
	\end{aligned}
	\]
	where the second inequality follows from $\alpha L_\Phi(\kappa,\tau)\le1$. The
	identity
	\[
	-\langle g,g+e\rangle+\frac12\|g+e\|^2
	=\frac12\bigl(\|e\|^2-\|g\|^2\bigr)
	\]
	therefore yields
	\begin{equation}
		\begin{split}
			\Phi_{\cR}^{\kappa,\tau}(w^{t+1}) \le{}&
			\Phi_{\cR}^{\kappa,\tau}(w^t)-\frac{\alpha}{2}\|g^t\|^2+
			\frac{\alpha}{2}\|e^t\|^2.
		\end{split}
		\label{eq:proof_descent}
	\end{equation}
	
	Rearranging \eqref{eq:proof_descent} and summing from $t=0$ to
	$T-1$ gives
	\[
	\frac{\alpha}{2}\sum_{t=0}^{T-1}\|g^t\|^2
	\le\Phi_{\cR}^{\kappa,\tau}(w^0)-\Phi_{\cR}^{\kappa,\tau}(w^T)
	+\frac{\alpha}{2}\sum_{t=0}^{T-1}\|e^t\|^2.
	\]
	Since $\|e^t\|\le\epsilon(\delta,\kappa)$, the accumulated error term is at most $T\epsilon(\delta,\kappa)^2$. Moreover, the nonnegativity of $\phi_{\kappa,\tau}$ and~\eqref{eq:fair_objective} imply $\Phi_{\cR}^{\kappa,\tau}(w^T)\ge0$. Therefore,
	\[
	\begin{aligned}
		\frac1T\sum_{t=0}^{T-1}\|g^t\|^2
		&\le\frac{2[\Phi_{\cR}^{\kappa,\tau}(w^0)-\Phi_{\cR}^{\kappa,\tau}(w^T)]}
		{\alpha T}+\epsilon(\delta,\kappa)^2\\
		&\le\frac{2\Phi_{\cR}^{\kappa,\tau}(w^0)}{\alpha T}
		+\epsilon(\delta,\kappa)^2.
	\end{aligned}
	\]
	Recalling the definition of $g^t$ proves
	\eqref{eq:stationarity_bound}.
	\hfill$\square$
	
	\fi
	
\end{document}